\documentclass{article}
\usepackage{spconf,amsmath,graphicx}
\usepackage{amsfonts}
\usepackage[caption=false,font=footnotesize]{subfig}
\usepackage[T1]{fontenc}
\usepackage{mathtools}   
\usepackage{algorithm}
\usepackage{algpseudocode}
\usepackage{xcolor}
\usepackage{amsthm}
\usepackage{hyperref}

\usepackage{graphicx}

\newcommand{\RR}{{\mathbb R}}
\newcommand{\be}{\begin{equation}}
\newcommand{\ee}{\end{equation}}

\newtheorem{thm}{Theorem}
\newtheorem{lemma}{Lemma}

\title{Generative Plug and Play:\\ 
Posterior Sampling for Inverse Problems}

\twoauthors{Charles A Bouman
\thanks{This work was partially supported by NSF grant number CCF-1763896.  CAB was partially supported by the Showalter Trust. The authors thank Katherine Bouman, Yu Sun, and Zihui Wu for their very helpful discussions.}}
{Purdue University, \\
ECE/BME, West Lafayette IN\\
bouman@purdue.edu}
{Gregery T Buzzard}
{Purdue University, \\
Mathematics, West Lafayette IN\\
buzzard@purdue.edu}

\begin{document}
\maketitle

\begin{abstract}
Over the past decade, Plug-and-Play (PnP) \cite{venkatakrishnan2013school,sreehari2016TCI} has become a popular method for reconstructing images using a modular framework consisting of a forward and prior model. The great strength of PnP is that an image denoiser can be used as a prior model while the forward model can be implemented using more traditional physics-based approaches.
However, a limitation of PnP is that it reconstructs only a single deterministic image.

In this paper, we introduce Generative Plug-and-Play (GPnP), a generalization of PnP to sample from the posterior distribution. As with PnP, GPnP has a modular framework using a physics-based forward model and an image denoising prior model. However, in GPnP these models are extended to become proximal generators, which sample from associated distributions. GPnP applies these proximal generators in alternation to produce samples from the posterior. 
We present experimental simulations using the well-known BM3D denoiser \cite{DabovBM3D07}. Our results demonstrate that the GPnP method is robust, easy to implement, and produces intuitively reasonable samples from the posterior for sparse interpolation and tomographic reconstruction. Code to accompany this paper is available at \url{https://github.com/gbuzzard/generative-pnp-allerton}.
\end{abstract}

\keywords{Plug and Play, prior modeling, inverse problems}

\section{Introduction}

The recent explosion in new sensors has led to growing interest in integrating both physical and data driven models for scientific applications.  This approach captures the enormous power of modern machine learning methods to model empirical data while also incorporating the benefits of established physics models in imaging applications ranging from optics \cite{PellizzariBate2023} to X-ray CT \cite{MajeeBalke2021}.

A popular method for integrating physics and machine learning models is Plug-and-Play (PnP) \cite{venkatakrishnan2013school}. 
The key idea behind PnP is that an image denoising algorithm encodes prior information implicitly and can be used in place of a functional prior model commonly found in Bayesian approaches. 
In practice, PnP alternates the application of a forward model proximal map to fit data and a denoiser representing the prior model. 
When the denoiser is also a proximal map, then PnP can be viewed as an optimization algorithm \cite{sreehari2016TCI}. However, more generally PnP is the solution to an equilibrium condition, and under appropriate technical conditions, the algorithm is known to converge to a unique \cite{Buzzard2018} solution.

The desire to understand variation in possible solutions given limited, noisy measurements has driven interest in algorithms to sample from the posterior distribution. Generative adversarial networks (GAN) \cite{Goodfellow2014} and variational autoencoders \cite{KingmaWelling2013} are two possible methods for generating samples from a distribution described by training data. However, while conditional GANs allow the samples to be conditioned on another random quantity \cite{MirzaOsindero2014}, neither model provides a modular framework that can be decomposed as a forward and prior model, and GANs can be difficult to stably train \cite{ArjovskyBottou2017}. 

More recently, generative diffusion methods \cite{SongErmonNeurIPS2019} based on denoising score matching (DSM) \cite{PVincent2011,SongErmonNeurIPS2020} and Langevin dynamics \cite{GrenanderMiller94} have displayed remarkable generative capabilities. These algorithms do not require adversarial training and have been reported to produce very high quality results \cite{SongPooleICLR2021}. A number of groups have investigated the use of these generative diffusion methods as a prior model that works along with a separate physics-based forward model \cite{FengBoumanFreeman2017,JalalTamir2021,SongShenXingErmon2022,KlaskyYe2023}.

In this paper, we introduce Generative Plug-and-Play (GPnP), a method for sampling from the posterior distribution of a model. As with PnP, GPnP has a modular framework based on a forward and prior model in which the prior model is implemented with a denoiser. The GPnP algorithm alternately applies a forward model and a prior model, each in the form of a proximal generator. These proximal generators are similar in formulation to a proximal map but generate random rather than deterministic outputs. 

Our primary theoretical result is a theorem that this alternating sequence of random functions generates a Markov chain (MC) with the desired stationary distribution. We then show how the methods of denoising score matching \cite{PVincent2011} can be used to approximate the prior proximal generator with a denoiser plus some AWGN. We also describe how to compute or approximate the forward model proximal generator in several common cases. 

We note that GPnP differs from generative diffusion methods in that it (a) formulates the solution as the stationary distribution of a discrete-time MC; (b) does not use a Langevin dynamics to generate the solution; (c) incorporates proximal generators rather then gradient updates. 
However, we do show that in the special case of a null forward model, GPnP generates an MC that is exactly the Langevin dynamics for generation of samples from a prior distribution. 

We present experimental simulations using the well-known BM3D denoiser \cite{DabovBM3D07}. These results demonstrate that the GPnP method is robust, easy to implement, and produces intuitively reasonable samples from the posterior for sparse interpolation and tomographic reconstruction.

\section{Generative PnP Theory}
\label{sec:GenerativePnP}

Let $u_0 (x)$ and $u_1(x)$ be two non-negative integrable energy functions; that is, $u_0, u_1 :\RR^p \rightarrow [0,\infty)$ and 
$$
Z = \int_{\RR^p} \exp \left\{ - u_1 (x) - u_0 (x) \right\} dx < \infty \ .
$$
Then our goal will be to generate samples from the distribution
\begin{equation}
\label{eq:CompoundDistribution}
p(x) = \frac{1}{Z } \exp \left\{ - u_1 (x) - u_0 (x) \right\} \ ,
\end{equation}
with the interpretation that $u_1$, $u_0$ are the energy functions for the data distribution and prior distribution, respectively.  

\subsection{Proximal Distributions and Generators}

To do this, we introduce the {\em proximal distributions} given by
\begin{align} 
q_0 (x|v) &= \frac{1}{Z_0(v)} \exp \left\{ - u_0 (x) - \frac{1}{2\gamma^2 } \| x-v \|^2 \right\} 
\label{eq:GProx0}\\
q_1 (x|v) &= \frac{1}{Z_1(v)} \exp \left\{ - u_1 (x) - \frac{1}{2\gamma^2 } \| x-v \|^2 \right\} 
\label{eq:GProx1}
\end{align}
where $\gamma$ is a parameter of the proximal distribution and again $Z_0(v)$ and $Z_1(v)$ are normalizing constants that depend on $v$. 
By assumption, $u_0\geq 0$ and $u_1\geq 0$, so the quadratic term implies that $Z_i (v)<\infty $ for all $v\in \RR^{N}$.

With the proximal distributions, we define {\em proximal generators} denoted by $F_0 (v)$ and $F_1 (v)$.
Intuitively, a proximal generator generates a new independent random variable with the proximal distribution.
More specifically, let 
\begin{align}
Y_0 &= F_0 (V) \\ 
Y_1 &= F_1 (V) \ ,
\end{align}
where $V$ is a random vector in $\RR^p$.
Then $Y_0$ and $Y_1$ are assumed conditionally independent of any previously generated random vectors given $V$,
and the conditional densities of $Y_0$ and $Y_1$ given $V$ are given above in~\eqref{eq:GProx0} and~\eqref{eq:GProx1}, respectively.

\subsection{Markov Chains from Proximal Generators}

We can produce a Markov chain (MC) by repeatedly applying the proximal generators.
More specifically, each new state of the MC is generated from the previous state by applying the two proximal generators in sequence.
\begin{align}
X_{n} &= F_1 ( F_0 (X_{n-1})) .
\label{eq:MC1}
\end{align}

Ideally, by repeatedly applying this sequence of operations, the random vector $X_n$ will converge in distribution to samples from $p(x)$.
This isn't quite true, but the following theorem, proved in the appendix, states that when $\gamma$ is small, then the MC has a stationary distribution near $p(x)$ in \eqref{eq:CompoundDistribution} with $u_0$ replaced by a Gaussian convolution approximation to $u_0$.
\begin{thm}
\label{th:OneState}

Let $X_n$ be a Markov chain given by
\begin{equation} \label{eq:MC-theorem1}
    X_n = F_1 ( F_0 (X_{n-1})) \ .
\end{equation}
Then $X_n$ forms a reversible Markov chain with a stationary distribution given by
\begin{equation} \label{eq:stationary-dist}
X_n \sim \tilde{p}_{\gamma^2} (x) = \frac{1}{Z^{\prime}} \exp \left\{ - u_1 (x) - \tilde{u}_0 (x ; \gamma^2) \right\} \ ,
\end{equation}
where
\begin{equation} \label{eq:u0-approx}
\tilde{u}_0 (x; \gamma^2) = -\log \left( e^{-u_0 (x) } * g_{\gamma^2} (x) \right) \ ,
\end{equation}
and $*$ denotes multidimensional convolution with a Gaussian density of variance $\gamma^2$ given by 
\begin{equation} \label{eq:gaussian}
g_{\gamma^2} (x) = \frac{1}{(2\pi \gamma^2 )^{p/2}} \exp \left\{ - \frac{1}{2\gamma^2} \| x \|^2 \right\} \ .
\end{equation}
\end{thm}

This theorem serves as the basis for the generative Plug-and-Play (GPnP) algorithm.
Assuming that the MC is ergodic, as $n\rightarrow \infty$, the GPnP algorithm will converge to the stationary distribution, $\tilde{p}_{\gamma^2 } (x)$.
Furthermore, this stationary distribution has the property that
\begin{eqnarray}
p (x) = \lim_{\gamma \rightarrow 0} \tilde{p}_{\gamma^2 } (x) \ ,
\label{eq:ApproxDistribution}
\end{eqnarray}
so the samples of the MC become close to the desired distribution as $n\rightarrow \infty$ and $\gamma \rightarrow 0$.

\section{Sampling from the Posterior}
\label{sec:SamplingPosterior}

In this section, we show how GPnP can be used to generate samples from the posterior distribution for a canonical inverse problem with data $y$ and object of interest $x$. 
Given a prior distribution $p_0(x)$ and a forward model $p_{y|x} (y|x)$, we define
\begin{align}
\label{eq:PriorEnergy}
u_0 (x) &= - \log p_0 (x) + C_0 \\
\label{eq:FowardEnergy}
u_1 (x) &= - \log p_{y|x} (y|x) + C_1 \ .
\end{align}
By Bayes' rule, the posterior distribution of $X$ given $Y$ can be expressed as
$$
p_{x|y} (x|y ) = \frac{1}{Z } \exp \left\{ -u_1(x) - u_0 (x) \right\} \ . 
$$
Note that this has the same form as \eqref{eq:CompoundDistribution}.
So Theorem~\ref{th:OneState} implies that the GPnP algorithm can be used to sample from the posterior distribution.

In order to implement the GPnP algorithm, we will need to implement both the forward proximal generator $F_1(v)$  and the prior proximal generator $F_0(v)$.

To implement the prior proximal generator, we use the recent theory of denoising score matching \cite{PVincent2011}.
This theory relates the MMSE denoiser for noise variance of $\sigma^2$ to a modified noisy prior distribution given by
$$
\tilde{p}_{0,\sigma^2} (x) = (p_0  * g_{\sigma^2 }) (x) \ ,
$$
which is a blurred version of the true prior distribution $p_0 (x)$. 
The associated energy function for $\tilde{p}_{0,\sigma^2}$ is then given by
$$
\tilde{u}_0 (x; \sigma^2 ) = - \log \tilde{p}_{0,\sigma^2} (x) + C_0 \ .
$$
As before, $\tilde{u}_0 (x; \sigma^2)$ is not exactly the desired energy function of $u_0(x)$, but as $\sigma \rightarrow 0$ it becomes a good approximation.
Hence we use this energy function to implement the prior proximal generator $\tilde{F}_0 (v ; \sigma )$ in the GPnP algorithm.

Note that the blur introduced from this $\sigma^2$-denoiser is independent from the noise introduced by $\gamma$ in the GPnP Algorithm as specified in Theorem~\ref{th:OneState}. 
This means that if we use $u_1$ and $\tilde{u}_0$ as the forward and prior energy functions, then GPnP will generate samples from the posterior distribution
$$
\tilde{p}_{x|y} (x|y; \sigma^2+\gamma^2 ) = p_{y|x} (y|x) \tilde{p}_{0,\sigma^2+\gamma^2 } (x) \ ,
$$
where $\tilde{p}_{0,\sigma^2+\gamma^2 } (x) =  (\tilde{p}_{0,\sigma^2} * g_{\gamma^2 } ) (x)$ is a version of the prior distribution that is blurred with a Gaussian of variance $\sigma^2+\gamma^2 $.
Again, as $\sigma$ and $\gamma$ become small, we get that
$$
p_{x|y} (x|y ) = \lim_{\sigma \rightarrow 0} \lim_{\gamma \rightarrow 0}  \tilde{p}_{x|y} (x|y; \sigma^2+\gamma^2 ) \ .
$$
So we can use the GPnP algorithm to generate samples from the true posterior distribution of $X$ given $Y$.

The following sections provide more details on how to implement the proximal generators $\tilde{F}_0 (v)$ and $F_1 (v)$.

\subsection{Prior Model Proximal Generator}
\label{sec:PriorModelProximalGenerators}

In this section, we show how to implement the proximal generator, $X=\tilde{F}_0(v)$ of the previous section.
We first define the score of the blurred distribution as 
\begin{equation} \label{eq:score}
s (x ; \sigma^2 ) = - \nabla_x \tilde{u}_0 (x; \sigma^2 ) \ .
\end{equation}
Vincent showed the amazing result that this score can be estimated by minimizing the Denoising Score Matching (DSM) loss~\cite{PVincent2011}.
For the special case of AWGN, the DSM has the form~\cite{SongErmonNeurIPS2020}, 
\begin{equation}
\mbox{Loss}(\theta ; \sigma ) = E\left[ \left\| \frac{W}{\sigma } + s_\theta (X + \sigma W ) \right\|^2 \right] \ ,
\label{eq:DSMLoss}
\end{equation}
where $s_\theta$ is a function parameterized by $\theta$, $X\sim p_0 (x)$ is a random image from the desired prior distribution, and \mbox{$W \sim N(0,I)$} is independent Gaussian white noise.

The key result of Vincent's work is that the loss in \eqref{eq:DSMLoss} is minimized when the function $s_{\theta_\sigma} ( x )$ is equal to the score, $s(x; \sigma^2 )$.
To best estimate the score of the blurred distribution, we choose $\theta$ to be
$$
\theta_\sigma = \arg \min_{\theta } \mbox{Loss}(\theta ; \sigma ) \ .
$$

A more traditional point of view is that \eqref{eq:DSMLoss} implies that the MMSE denoiser with AWGN of variance $\sigma^2$ is given by
\begin{equation} \label{eq:denoise}
\mbox{Denoise} (x; \sigma ) = x + \sigma^2 s_{\theta_\sigma} (x ) \ .
\end{equation}
From this, we see that if we have an MMSE denoiser designed for a noise variance of $\sigma^2$, then we can compute an estimate of the score as
\begin{equation}
s_{\theta_\sigma} (x ) = \frac{1}{\sigma^2 } \left[ \mbox{Denoise} (x; \sigma ) - x \right] \ .
\label{eq:ScoreInTermsOfDenoiser}
\end{equation}

Then a first order Taylor series and completing the square yields an approximate proximal distribution given by
\begin{align} 
\nonumber
&\tilde{q}_0 (x|v; \sigma^2 ) \\
\nonumber
&= \frac{1}{Z(v) } \exp \left\{ - \tilde{u}_0 (x; \sigma^2 ) - \frac{1}{2\gamma^2} \| x-v \|^2 \right\} \\
\nonumber
&\approx \frac{1}{Z^\prime (v)} \exp \left\{ (x-v)^t s_{\theta_\sigma } (v)- \frac{1}{2\gamma^2} \| x-v \|^2 \right\} \\
&= \frac{1}{Z^{\prime \prime} (v)} \exp \left\{ - \frac{1}{2\gamma^2} \| x- [v+\gamma^2 s_{\theta_\sigma } (v) ] \|^2 \right\} \ .
\label{eq:ProximalDistributionApproximation}
\end{align} 
Notice that for this approximation to be accurate, we need that $\gamma << \sigma$ so that the second derivative of the score function is small relative to $1/\gamma^2$.
In order to ensure this, we will express our results in terms of $\beta = \gamma^2/\sigma^2$, where we will pick the parameter $\beta<1$.

Combining \eqref{eq:ScoreInTermsOfDenoiser} and \eqref{eq:ProximalDistributionApproximation}, we can rewrite the proximal generator as
\begin{align}
\tilde{F}_0 (v; \beta, \sigma ) 
&\approx (1-\beta)  v + \beta \, \mbox{Denoise} (v; \sigma ) + \sqrt{\beta} \sigma W \ ,
\label{eq:F0Update}
\end{align}
where $W\sim N(0,I)$, $\beta < 1$, and $\mbox{Denoise} (v, \sigma)$ is an MMSE denoiser designed to remove AWGN of variance $\sigma^2$.

\subsection{Forward Model Proximal Generator}
\label{sec:ForwardModelProximalGenerator}

We first consider the case in which $u_1(x)$ has two continuous derivatives.
In this case, we denote the proximal map for $u_1$ as
\begin{align}
\label{eq:F1bar}
\bar{F}_1(v; \gamma ) 
    &= \arg \min_{x\in \RR^p} \left\{ u_1 (x)  + \frac{1}{2 \gamma^2 } \| x-v \|^2 \right\} \ .
\end{align}
Again, a first order approximation for $u_1$, this time centered at the proximal point $\bar{F}_1(v; \gamma)$, implies that for $\gamma$ small, we can express the proximal distribution as
$$
q_1(x|v; \gamma ) \approx \frac{1}{Z} \exp \left\{ -\frac{1}{2\gamma^2 } \| x- \bar{F}_1(v; \gamma ) \|^2 \right\} \ .
$$
So then for small $\gamma$, the forward model proximal generator can be implemented as 
\begin{align}
F_1 (v; \gamma ) \approx \bar{F}_1(v; \gamma ) + \gamma\, W  \ ,
\label{eq:GeneralForwardModelProximalGenerator}
\end{align}
where $W\sim N(0,I)$.

From this we see that for sufficiently small values of $\gamma$, we can approximate the forward model proximal generator as the forward model proximal map plus Gaussian white noise.
However, in some cases we can practically implement a more accurate proximal generator for larger values of $\gamma$ as discussed in Sections~\ref{sec:LinearForwardModelProximalGenerator} and~\ref{sec:SamplingForwardModelProximalGenerator}.

\begin{figure}
\framebox[3.4in]{
\hspace*{-15pt}
\parbox{3.4in}{
\small
\begin{algorithmic}
\Procedure{GPnP-basic}{$\alpha, \beta , \sigma_{max}, \sigma_{min}, N$ }
    \State $X \gets \sigma_{max} \mbox{RandN}(0,I)+1/2$
    \State $a \gets \left( \frac{\sigma_{min}}{\sigma_{max}} \right)^{1/N}$
    \For{$n=0$ to $N-1$}
        \State $\sigma \gets a^n  \sigma_{max}$
        \State $X \gets (1-\beta ) X + \beta \, \mbox{Denoise} (X; \alpha \sigma) + \sqrt{\beta} \sigma \, \mbox{RandN}(0,I)$
        \State $X \gets \bar{F}_1 (X; \sqrt{\beta} \sigma ) + \sqrt{\beta} \sigma \, \mbox{RandN}(0,I)$
    \EndFor
    \State \textbf{return} $X$
\EndProcedure
\end{algorithmic}
}}
\caption{Generative Plug-and-Play where $\alpha \in[1,1.5]$, $0<\beta <<1$, $\sigma$ decreases by the multiplicative factor $a$ on each iteration, $\bar{F}_1$ denotes the forward model proximal map, and $\mbox{Denoise} (x; \sigma)$ performs MMSE denoising of an image $x$ with AWGN of variance $\sigma^2$.}
\label{al:GPnP-Main}
\end{figure}

\subsection{The GPnP Algorithm}

Algorithm~\ref{al:GPnP-Main} provides a pseudo-code implementation of the GPnP algorithm that starts with a large value of $\sigma$ and then iterates the GPnP proximal generators $\tilde{F}_0$ and $F_1$ for each value of $\sigma$.
Notice that the prior proximal generator, $\tilde{F}_0$, is implemented with the approximation of \eqref{eq:F0Update},
and the forward proximal generator, $F_1$, is implemented as described in \eqref{eq:GeneralForwardModelProximalGenerator} with $\gamma= \sqrt{\beta}\sigma $.

The decreasing sequence of $\sigma$ values is known as annealing and has been shown to dramatically speed convergence to the stationary distribution of the Markov chain by more stably modeling the distribution in low probability regions of the space \cite{SongErmonNeurIPS2019}.
In our experiments, we have found that $\beta=0.25$ works well.
We also incorporate a parameter $\alpha$ to modulate the strength of the denoiser.  This is used to account for inaccuracies in denoiser calibration and to account for the fact that $\sigma$ decreases on each iteration, which means we need to denoise at a higher rate than the current value of $\sigma$.  We use $\alpha\approx 1.3$ in our experiments.

\section{Special Proximal Generators}
\label{sec:SpecialForwardModelProximalGenerators}

In this section, we discuss some special cases of proximal generators that can be useful in practice.

\subsection{Proximal Generator: Linear Forward Model }
\label{sec:LinearForwardModelProximalGenerator}

In this section, we show how to implement the proximal generator, $X=F_1(v)$, where $X\sim q_1 ( x |v) $, for a general linear forward model.
To do this, consider a linear forward model of the form
\begin{equation}  
\label{eq:ForwardModel}
Y = A X + W \ ,  
\end{equation}
where $W\sim N( 0, \Lambda^{-1})$, $A$ is a linear forward operator, and $\Lambda$ is a positive definite precision matrix.
Then the energy function associated with this forward model is given by
\begin{equation}  \label{eq:quadui}
    u_1(x) = \frac{1}{2} \| y- Ax \|_\Lambda^2 \ .  
\end{equation}
The first-order optimality conditions imply that the proximal map for this function is given by
\begin{align}
\bar{F}_1(v; \gamma ) = v + \left( A^T \Lambda A + \frac{1}{\gamma^2 } I \right)^{-1} A^t \Lambda (y - A v) \ .
\label{fig:GeneralLinearProxMap}
\end{align}
We define $R$ to be the conditional covariance given by
\begin{equation}
R = \left( A^T \Lambda A + \frac{1}{\gamma^2 }I \right)^{-1} \label{eq:R} \ .
\end{equation}
The energy function for the proximal distribution for $u_1$ is the objective in \eqref{eq:F1bar}, which has a minimum at the conditional mean $\bar{F}_1(v; \gamma )$ and which has Hessian $R$.  Since the energy function is quadratic, the proximal distribution is exactly
$$
q_1(x|v; \gamma ) = \frac{1}{Z} \exp \left\{ -\frac{1}{2} \| x- \bar{F}_1(v; \gamma ) \|^2_{R^{-1}} \right\} \ .
$$
From this, we see that the proximal generator is given by
$$
F_1 (v; \gamma ) = \bar{F}_1(v; \gamma ) + W_R \ ,
$$
where $W_R\sim N(0,R)$.

\begin{figure*}[ht]
\centering
\includegraphics[width=0.23\textwidth]{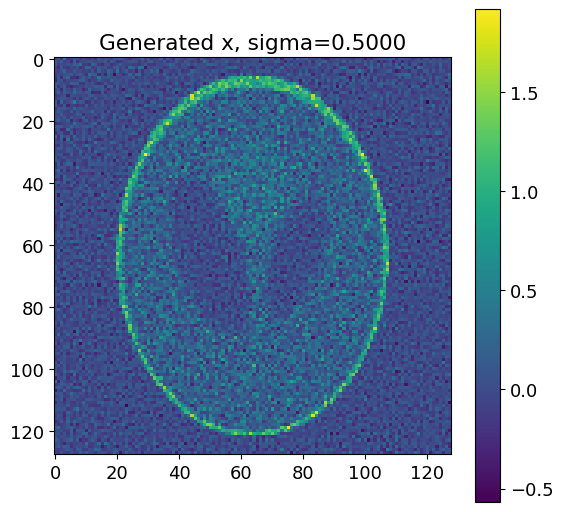}
\includegraphics[width=0.23\textwidth]{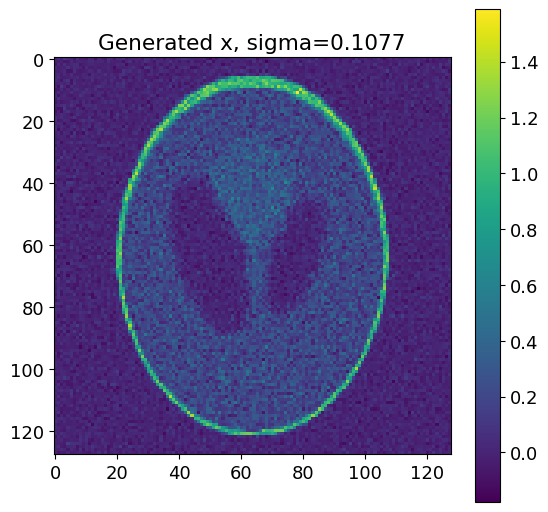}
\includegraphics[width=0.23\textwidth]{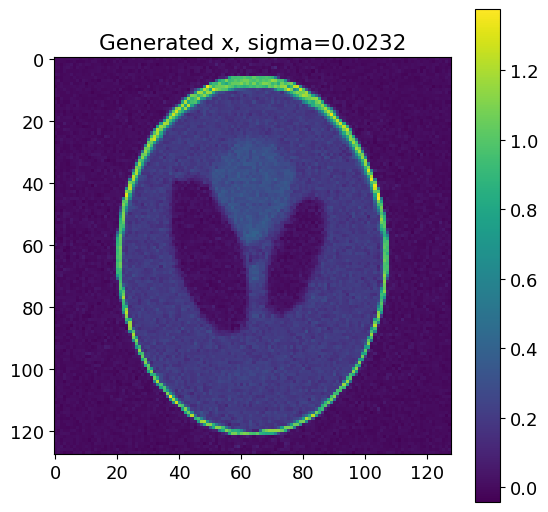}
\includegraphics[width=0.23\textwidth]{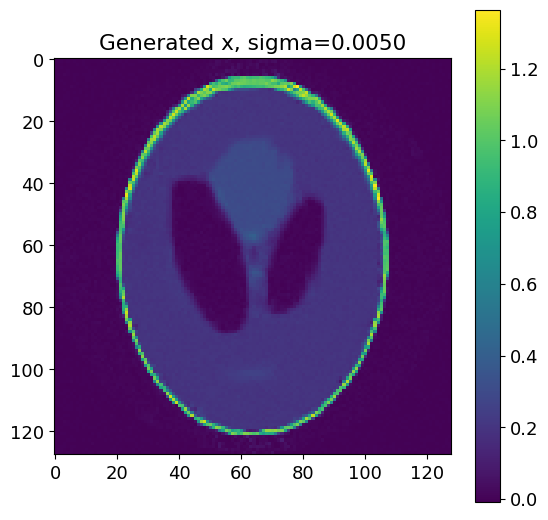}
\hspace*{0.75in}(a) $\sigma=0.500$\hspace*{\fill}(b) $\sigma=0.107$\hspace*{\fill}(c) $\sigma=0.023$\hspace*{\fill}(d) $\sigma=0.005$\hspace*{0.75in}
\newline
\hspace*{\fill} \includegraphics[width=0.25\textwidth]{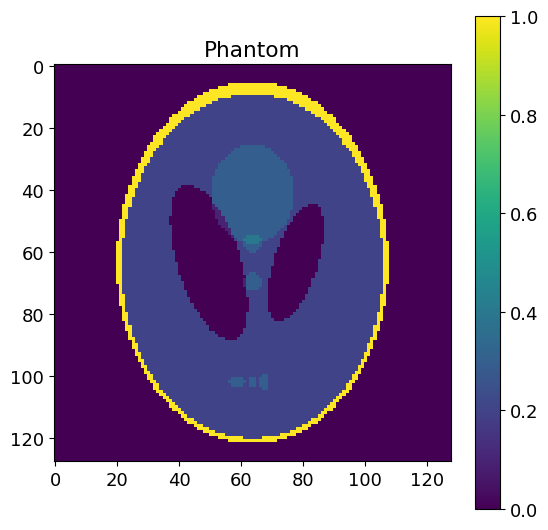} \hspace*{\fill}
\includegraphics[width=0.25\textwidth]{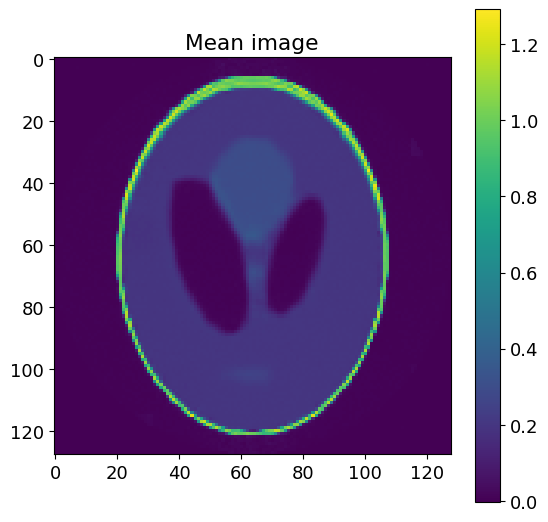}  \hspace*{\fill}
\includegraphics[width=0.25\textwidth]{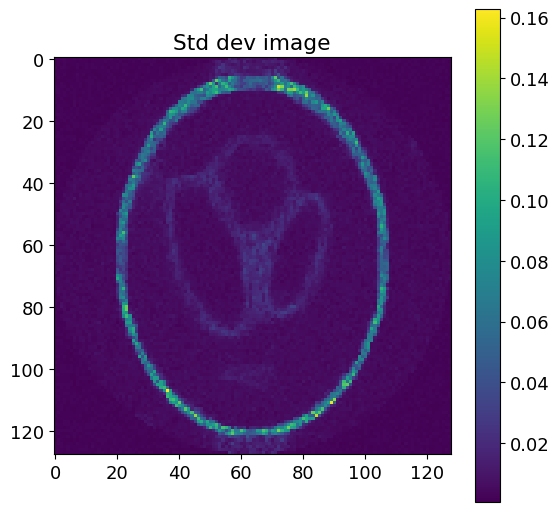} \hspace*{\fill}
\newline
\hspace*{0.75in}(e) Ground Truth\hspace*{\fill}(f) Mean over 10 samples\hspace*{\fill}(g) Std Dev over 10 samples\hspace*{0.5in}
\caption{Tomographic experiment for $128\times 128$ phantom with 16 views: 
(a) - (d) GPnP outputs with decreasing values of $\sigma$;
(e) phantom; (f) mean and (g) standard deviation over 10 trials.
}
\label{fig:GPnPTomographyAnneal}
\end{figure*}

\begin{figure*}[ht]
\centering
\includegraphics[width=0.23\textwidth]{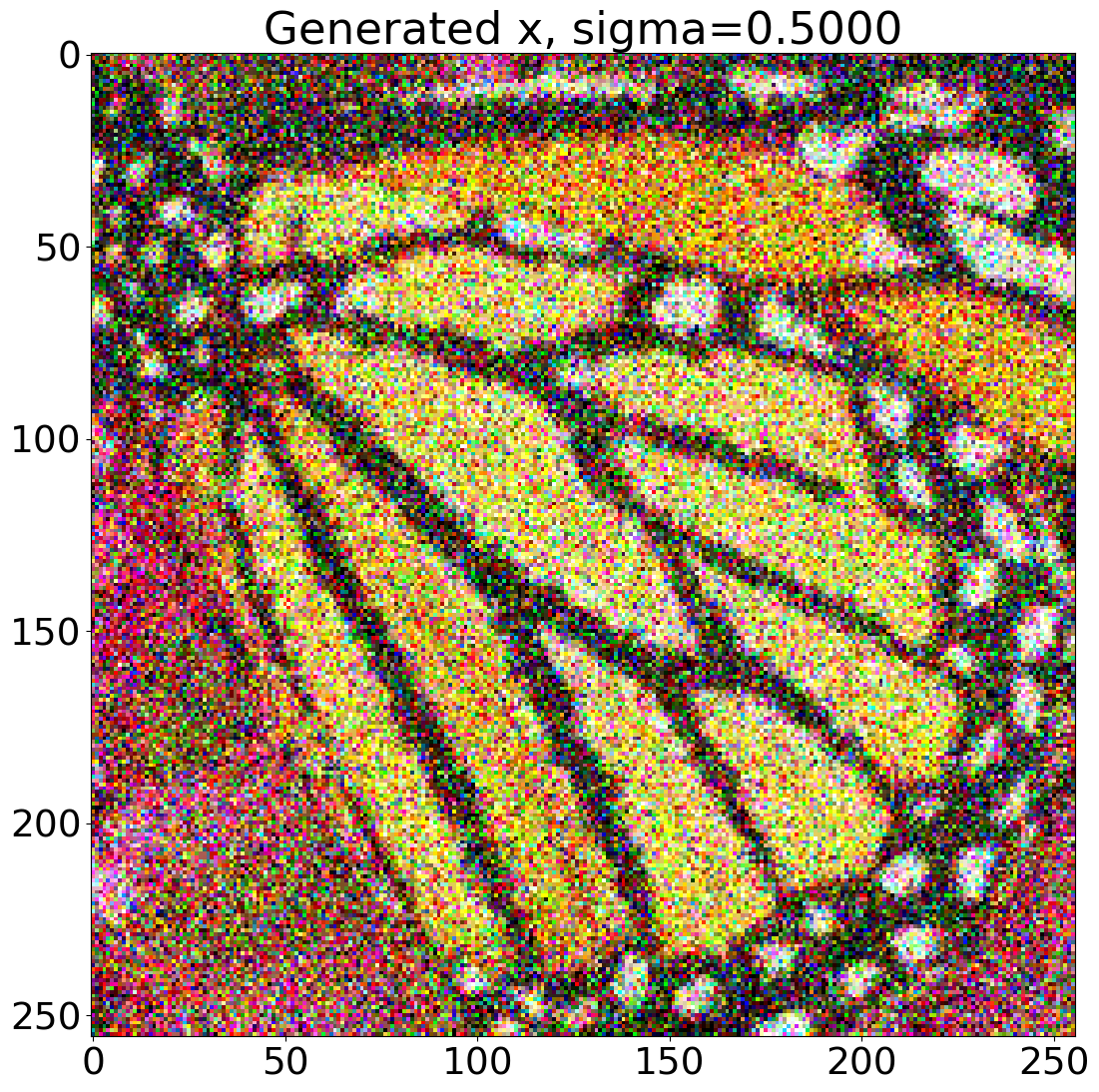}
\includegraphics[width=0.23\textwidth]{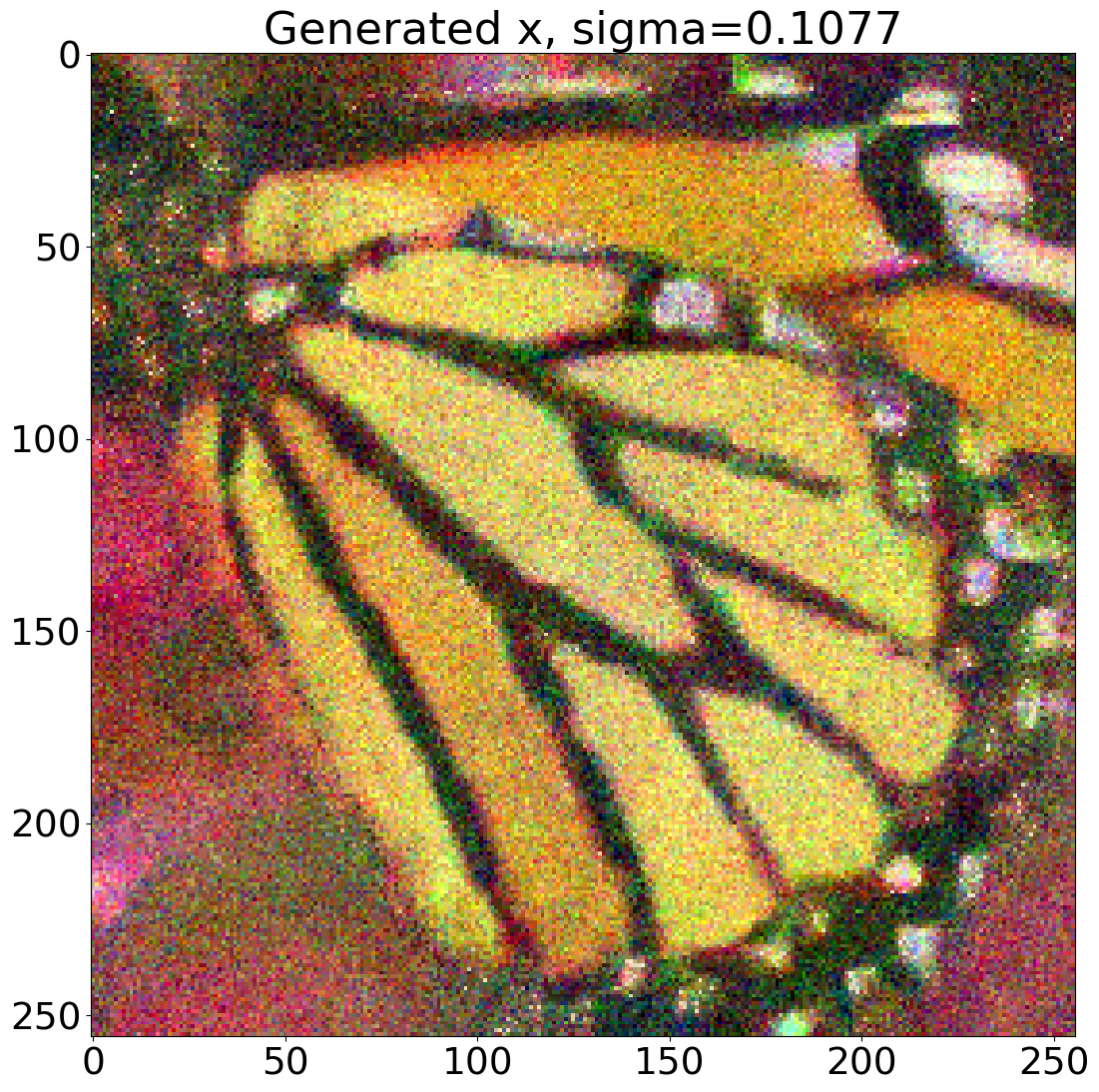}
\includegraphics[width=0.23\textwidth]{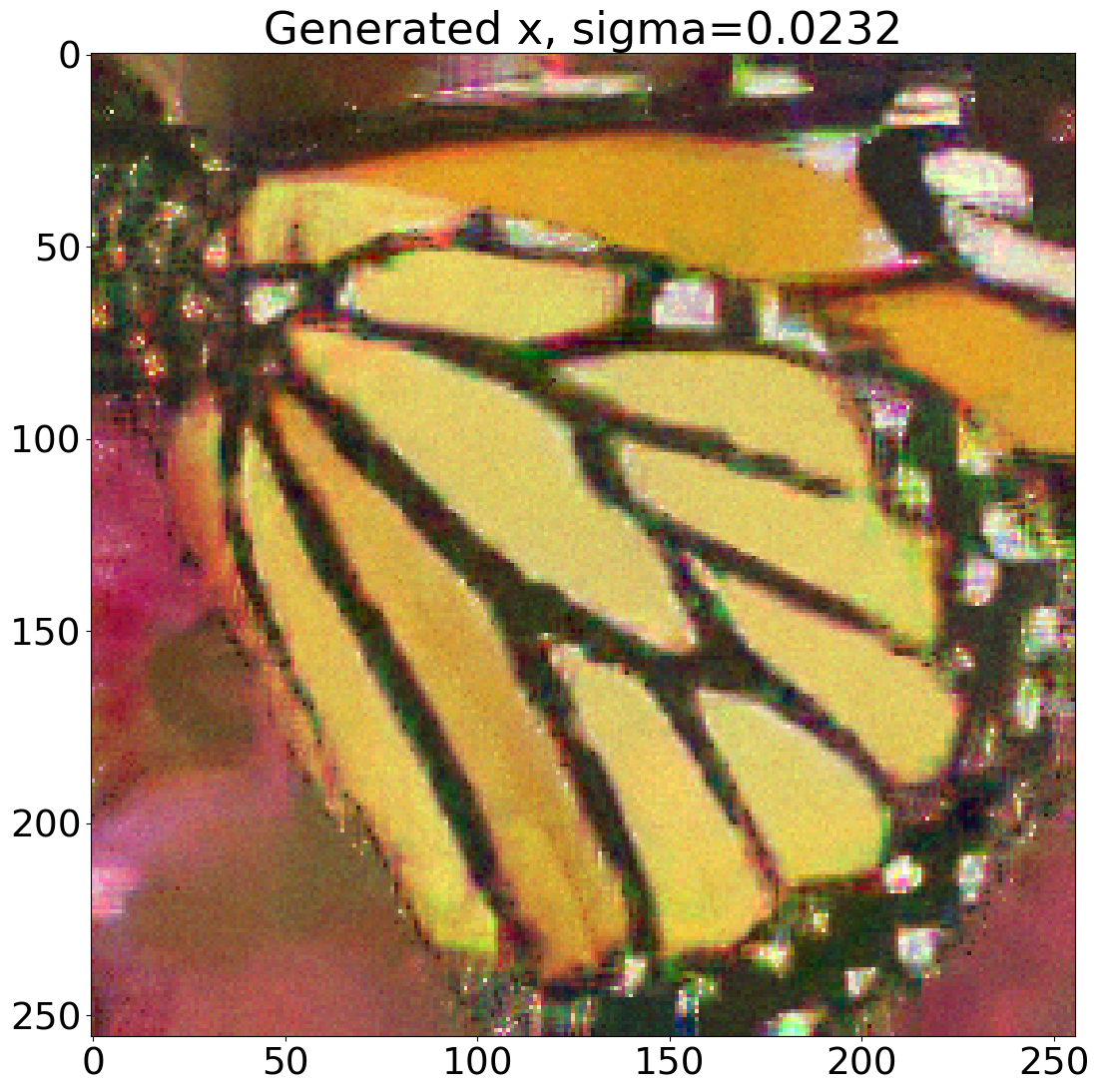}
\includegraphics[width=0.23\textwidth]{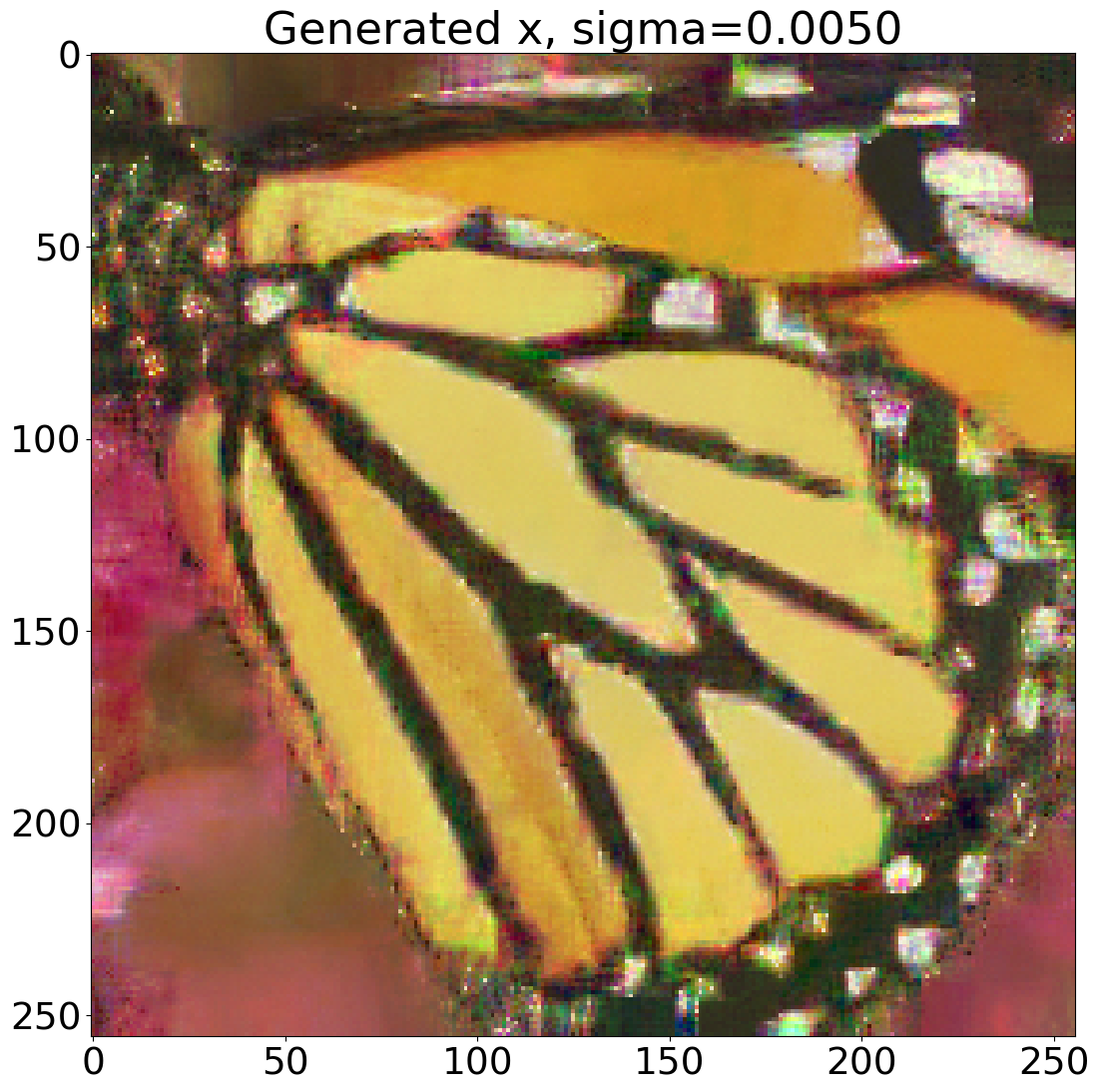}
\hspace*{0.75in}(a) $\sigma=0.500$\hspace*{\fill}(b) $\sigma=0.108$\hspace*{\fill}(c) $\sigma=0.023$\hspace*{\fill}(d) $\sigma=0.005$\hspace*{0.75in}
\newline
\hspace*{\fill} \includegraphics[width=0.25\textwidth]{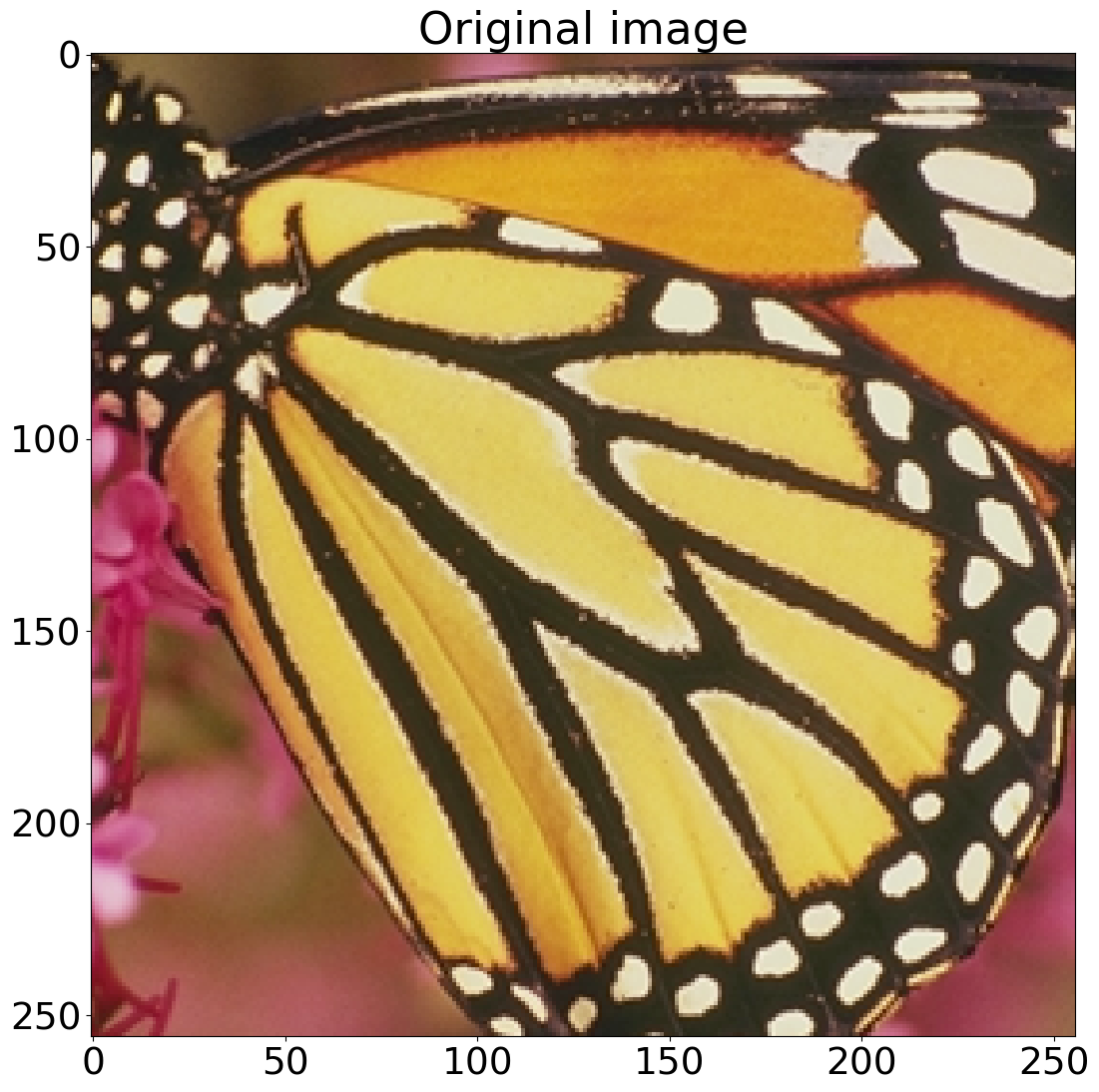} \hspace*{\fill}
\includegraphics[width=0.25\textwidth]{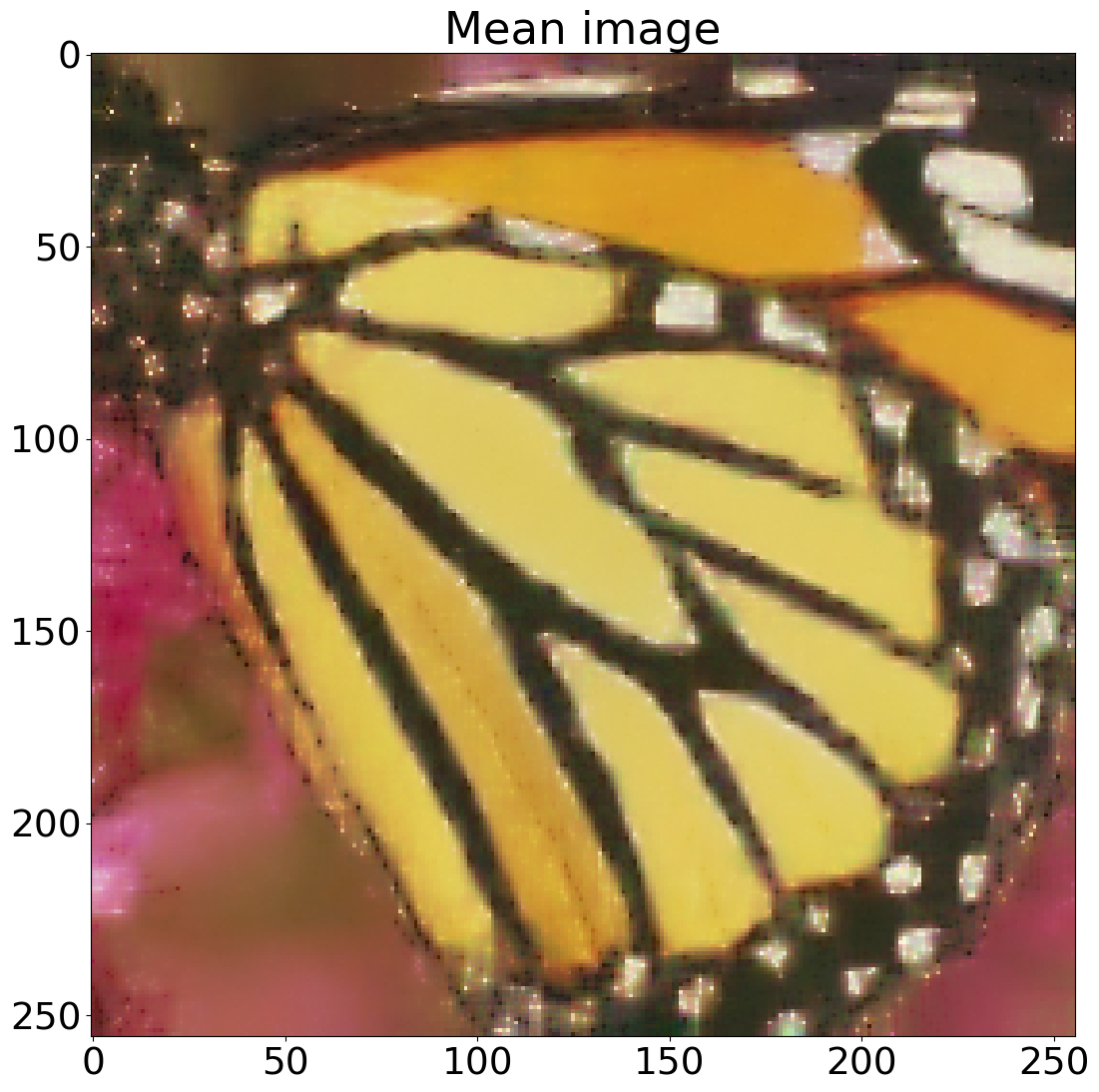}  \hspace*{\fill}
\includegraphics[width=0.25\textwidth]{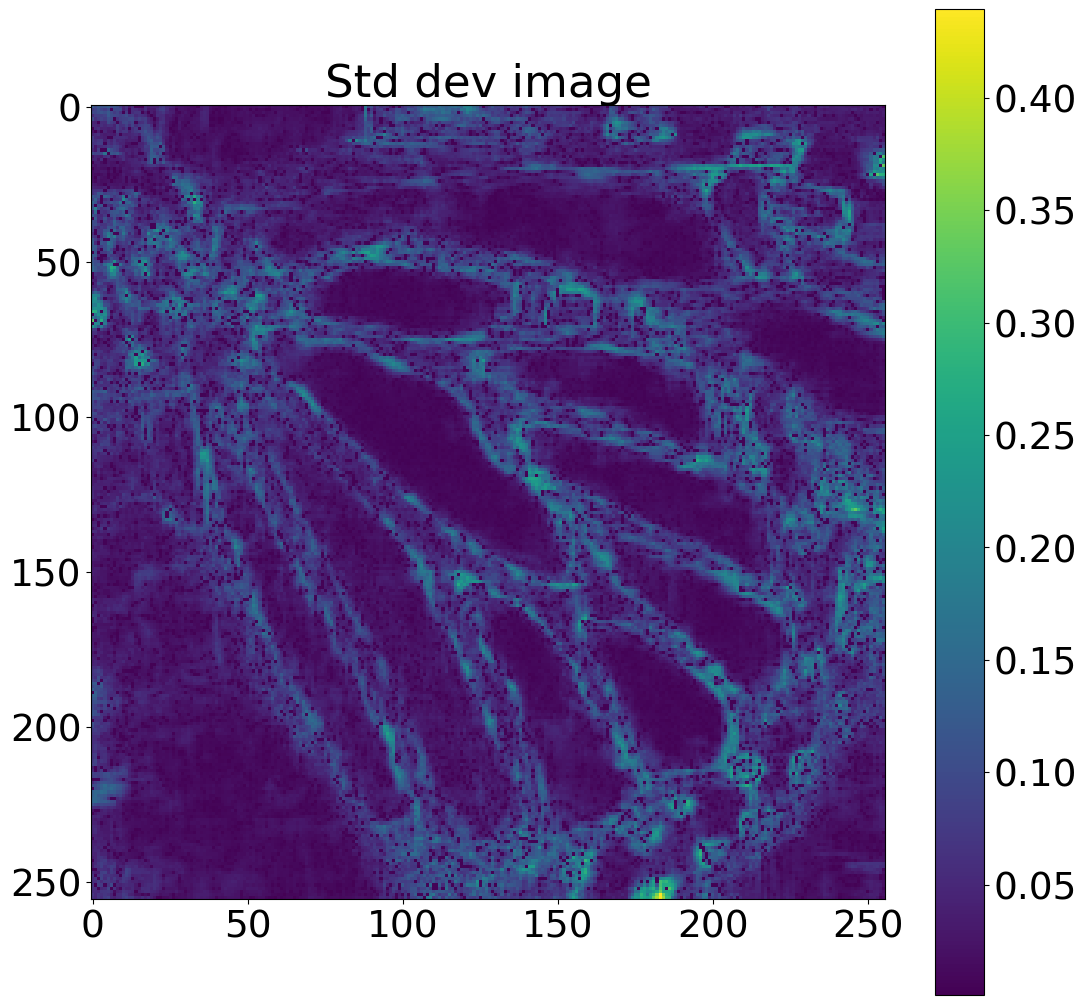} \hspace*{\fill}
\newline
\hspace*{0.75in}(e) Ground Truth\hspace*{\fill}(f) Mean over 10 samples\hspace*{\fill}(g) Std Dev over 10 samples\hspace*{0.4in}
\caption{Interpolation experiment for $256\times 256$ RGB image with 10\% of pixels sampled uniformly at random, starting from radial basis function interpolation plus AWGN with $\sigma=0.5$: 
(a) - (d) GPnP outputs with decreasing values of $\sigma$;
(e) original image; (f) mean and (g) standard deviation over 10 trials.
}
\label{fig:GPnPInterpolationAnneal}
\end{figure*}

\begin{table}[ht]
    \centering
    {\small
    \begin{tabular}{|c|c|c|c|c|c|c|} \hline
                        & $N$   & $\sigma_{max}$ & $\sigma_{min}$ & $\beta$ & $\alpha$ & $\sigma_y$ \\ \hline 
         Subsampling & 100 & 0.5 & 0.005 & 0.25 & 1.3 & 0.005 \\ \hline
         Tomography     & 100 & 0.5 & 0.005 & 0.25 & 1.3 & 0.25  \\ \hline 
    \end{tabular}
    }
    \caption{Parameters used in experiments}
    \label{tab:ListOfParameters}
\end{table}

\subsection{Proximal Generator: Subsampling }
\label{sec:SamplingForwardModelProximalGenerator}

Another useful special case occurs when our measurements are samples at selected pixels.
Let $S_m$ be a set of measurement points so that
$$
Y_s = X_s + W_s \ ,
$$
where $W_s \sim N(0, \sigma_y^2 )$ are i.i.d.\ noise samples.
In this case, the energy function associated with this forward model is given by
\begin{equation}  
\label{eq:SparseSampleEnergy}
u_1(x) = \sum_{s \in S_m} 
    \frac{1}{2 \sigma_y^2 } ( y_s - x_s )^2 \ .
\end{equation}
The first-order optimality conditions imply that the associated proximal map is given by 
\begin{align}
\label{eq:SparseSampleProxMap}
\left[ \bar{F}_1(v; \gamma ) \right]_s =
\left\{ 
\begin{array}{ll}
v_s + \frac{\gamma^2}{ \sigma_y^2 + \gamma^2 } (y_s - v_s) & \mbox{if $s\in S_m$} \\
v_s & \mbox{if $s\notin S_m$} \ .
\end{array}
\right.
\end{align}
From the result of Section~\ref{sec:LinearForwardModelProximalGenerator} the forward model proximal generator is given by
$$
\left[ F_1(v; \gamma ) \right]_s =
\left\{ 
\begin{array}{ll}
\left[ \bar{F}_1(v;\gamma ) \right]_s + \sqrt{ \frac{\sigma_y^2 \gamma^2}{\sigma_y^2 + \gamma^2 }} W_s & \mbox{if $s\in S_m$} \\
\left[ \bar{F}_1(v; \gamma ) \right]_s + \gamma W_s& \mbox{if $s\notin S_m$} 
\end{array}
\right. \ .
$$
where $W_s \sim N(0,1)$ are i.i.d.\ Gaussian random variables.

\subsection{Sampling from the Prior}
\label{sec:SamplingPrior}

Below we derive the update equations for sampling from the prior distribution.
In this case, we set $u_1 (x) = 0$, so from \eqref{eq:GeneralForwardModelProximalGenerator}, the forward proximal generator is exactly
$$
F_1 (v; \gamma ) = v + \gamma W \ ,
$$
where $W\sim N(0,I)$.  Then for small $\gamma$, \eqref{eq:denoise}, \eqref{eq:F0Update}, and the definition of $\beta$ imply that the prior model proximal generator is 
$$
\tilde{F}_0 ( v ) = v + \gamma^2 s_{\theta_\sigma } (v) + \gamma W^\prime \ ,
$$
where $W^\prime \sim N(0,I)$ is independent of $W$. 
Taking the composition of $\tilde{F}_0$ followed by $F_1$ results in the update 
$$
X_n = X_{n-1} + \gamma^2 s_{\theta_\sigma } (X_{n-1}) + \sqrt{2} \gamma W \ ,
$$
which is the familiar Langevin update equation \cite{GrenanderMiller94}.
Rewriting in terms of the denoiser and the parameters $\beta,\sigma$ results in the following recursion that generates samples from the posterior distribution for small $\sigma$
\begin{equation}
X_n = (1- \beta ) X_{n-1} + \beta \, \mbox{Denoise} (X_{n-1}; \sigma) + \sqrt{2 \beta } \sigma W \ ,
\label{eq:LangevinDenoise}
\end{equation}
where $W\sim N(0,I)$.

\section{Results}
\label{sec:Results}

In this section, we present experimental results using the GPnP algorithm to sample from the posterior distribution of a model.
We consider the cases of sparse image interpolation from Section~\ref{sec:SamplingForwardModelProximalGenerator} and 2D parallel-beam, sparse-view tomographic reconstruction from Section~\ref{sec:LinearForwardModelProximalGenerator}.
Table~\ref{tab:ListOfParameters} lists the parameters used for both experiments. 
For both experiments, the BM3D denoiser \cite{DabovBM3D07} was used as an implicit prior model.
However, we have found that more advanced, domain-specific denoisers such those used in \cite{SongErmonNeurIPS2020} can yield better results.

Figure~\ref{fig:GPnPTomographyAnneal} shows the results for the case of tomographic reconstruction using 8 views of a $128\times 128$ phantom. The algorithm was implemented using the SVMBIR tomographic software package \cite{svmbir-2020}.
Figures~\ref{fig:GPnPTomographyAnneal}(a) to (d) show a typical progression of samples for the GPnP algorithm as $\sigma$ decreases. For large values of $\sigma$, the prior is essentially white noise, so the reconstructed image has similar attributes. 
As $\sigma$ decreases, the image stabilizes to a less noisy image, but each trial produces a somewhat different result that represents the variation in the posterior distribution.
Figures~\ref{fig:GPnPTomographyAnneal}(f) and (g) show the mean and standard deviation over 10 trails. 
Notice that Figure~\ref{fig:GPnPTomographyAnneal}(g) shows that most of the variation occurs near edges, which is what one might expect.

Figure~\ref{fig:GPnPInterpolationAnneal} shows similar results for sparse interpolation from $10\%$ of the pixels from a color $256\times 256$ ground-truth image. 
This gives results qualitatively similar to the tomography case, with most of the variation occurring along image edges.

\section{Conclusion}
\label{sec:Conclusion}

In this paper, we presented a novel theory for Generative PnP, a generalization of the PnP that allows for sampling from the posterior distribution given a forward model and a prior specified using a MMSE denoising algorithm. 
As with PnP, GPnP has a modular implementation in which two proximal generators are alternately applied.
The proximal generators generate conditionally independent random variables from a distribution inspired by the proximal map and in practice can be implemented by adding noise to the conventional proximal map.

Our key theoretical result is that the sequence generated by GPnP forms a reversible Markov chain with the desired posterior distribution.
Our experimental results indicate that the algorithm can be robustly implemented for simple inverse problems such as sparse interpolation and 2D parallel beam tomographic reconstruction.

\appendix

\section{Proofs}

We first prove the following lemma
\begin{lemma} 
Let $[X_{n,0},X_{n,1}]$ be a Markov chain such that
\begin{align*}
V_n &= F_0 ( X_{n-1,1} ) \\
[X_{n,0},X_{n,1}] &= [ V_n , F_1 ( V_n ) ] \ .
\end{align*}
Then $[X_{n,0},X_{n,1}]$ is a Markov chain in time, $n$, with a stationary distribution given by
$$
[X_{n,0},X_{n,1}] \sim p (x_0,x_1)
$$
where
$$
p (x_0,x_1) = \frac{1}{Z} \exp \left\{- u_1 (x_1) - u_0(x_0) - 
\frac{1}{2\gamma^2} \| x_0 -x_1 \|^2 \right\} \ .
$$
\label{th:TwoState}
\end{lemma}

\begin{proof}[Proof of Lemma~\ref{th:TwoState}]

Note that the distribution $p (x_0,x_1)$ has conditional distributions 
\begin{eqnarray*}
p_{0|1} (x_0|x_1) &=& q_0 ( x_0| x_1 ) \\
p_{1|0} (x_1|x_0) &=& q_1 ( x_1| x_0 ) \ .
\end{eqnarray*}
Hence the Markov chain is an implementation of a Gibbs sampler that first replaces $X_{k,0}$ with a conditionally independent random variable from its conditional distribution, and then replaces $X_{k,1}$ with a conditionally independent random variable from its conditional distribution~\cite{GemanGeman1984,BoumanBook2022}.

For notational compactness, we denote the state at time $n-1$ by $(x_0 ,x_1)$, and the state at time $n$ by $(x_0^\prime ,x_1^\prime )$, and let $q (x_0^\prime ,x_1^\prime ) $ denote the distribution of the state at time $n$.
Defining $p_1(x_1) = \int p(x_0, x_1) dx_0$, we have $p(x_0, x_1) = q_0(x_0 | x_1) p_1(x_1)$.  Then standard manipulations give
\begin{align*}
q (x_0^\prime, &x_1^\prime )
= \int \int q_1 ( x_1^\prime | x_0^\prime ) q_0 ( x_0^\prime | x_1 ) p(x_0,x_1) dx_0 dx_1 \\
&= \int \int q_1 ( x_1^\prime | x_0^\prime ) q_0 ( x_0^\prime | x_1 ) q_0 ( x_0 | x_1 ) p_1(x_1) dx_0 dx_1 \\
&= \int q_1 ( x_1^\prime | x_0^\prime ) q_0 ( x_0^\prime | x_1 ) \int q_0 ( x_0 | x_1 ) dx_0 \, p_1(x_1) dx_1 \\
&= \int q_1 ( x_1^\prime | x_0^\prime ) q_0 ( x_0^\prime | x_1 ) p_1(x_1) dx_1 \\
&= \int q_1 ( x_1^\prime | x_0^\prime ) p ( x_0^\prime , x_1) dx_1.
\end{align*} 
Since $q_1 ( x_1^\prime | x_0^\prime )$ does not depend on $x_1$, while 
$p ( x_0^\prime , x_1) dx_1$ integrates to $p_0 ( x_0^\prime )$, this simplifies to give 
$$q_1 ( x_1^\prime | x_0^\prime ) p_0 ( x_0^\prime ) = q (x_0^\prime, x_1^\prime ) = p ( x_0^\prime , x_1^\prime ).$$ 
Hence $p ( x_0 , x_1 )$ is a stationary distribution of the Markov chain.
\end{proof}

\begin{proof}[Proof of Theorem~\ref{th:OneState}]
Recall that $X_n = F_1(F_0(X_{n-1})$ from Theorem~\ref{th:OneState}, so that $X_n$ is a Markov chain that equals $X_{n,1}$ in Lemma~\ref{th:OneState}.     
By Lemma~\ref{th:TwoState} we know that $X_{n,1}$ has a stationary distribution given by
\begin{align*}
p(&  x_1 ) = \int p( x_0 , x_1 ) d x_0 \\
    &= \int \frac{1}{Z } \exp \left\{- u_1 (x_1) - u_0(x_0) - \frac{1}{2\gamma^2} \| x_0 -x_1 \|^2 \right\} d x_0 \\
    &= \frac{1}{Z } \exp \left\{- u_1 (x_1) \right\} \cdot \\
    & \rule{20mm}{0pt}  \int \exp \left\{ - u_0(x_0) - \frac{1}{2\gamma^2 } \| x_0 -x_1 \|^2 \right\} d x_0 \ .
\end{align*}
Then notice that
\begin{align*}
\int \exp & \left\{ - u_0(x_0) - \frac{1}{2\gamma^2} \| x_0 -x_1 \|^2 \right\} d x_0 \\
&= \int e^{- u_0(x_0)} \exp \left\{ - \frac{1}{2\gamma^2} \| x_0 -x_1 \|^2 \right\} d x_0 \\
&= \left( e^{- u_0(\cdot)} * \exp \left\{ - \frac{1}{2\gamma^2} \| \cdot \|^2 \right\} \right)(x_1)\\
&= (2 \pi \gamma^2 )^{p/2} \left( e^{- u_0(\cdot)} * g_{\gamma^2} \right) (x_1 ) \\
&= (2 \pi \gamma^2 )^{p/2} \exp \left\{ - \tilde{u}_0 ( x_1 ) \right\} \ ,
\end{align*}
where 
$$
\tilde{u}_0  = -\log \left( e^{-u_0 } * g_{\gamma^2}  \right) \ .
$$
So we have that
\begin{eqnarray*}
p( x_1 ) 
    &=& \frac{1}{Z^{\prime }} \exp \left\{- u_1 (x_1) \right\} \exp \left\{ - \tilde{u}_0 ( x_0 ) \right\} \\
    &=& \frac{1}{Z^{\prime }} \exp \left\{- u_1 (x_1) - \tilde{u}_0 ( x_0 ) \right\} \ ,
\end{eqnarray*}
where $Z^\prime = Z / (2 \pi \gamma^2 )^{p/2}$.

To show reversibility, we denote the joint distribution of $(X_{n,1}, X_{n-1,1})$ as $q (x_1^\prime ,x_1 )$.  As in Lemma~\ref{th:TwoState}, we define $p_0(x_0) = \int p(x_0, x_1) dx_1$ and note that $p(x_0, x_1) = q_1(x_1 | x_0) p_0(x_0)$.  Then we have 
\begin{align*}
q (x_1^\prime , &x_1 ) = \int \int q_1 ( x_1^\prime | x_0^\prime ) q_0 ( x_0^\prime | x_1 ) p(x_0,x_1) dx_0 dx_0^\prime \\
=& \int \int q_1 ( x_1^\prime | x_0^\prime ) q_0 ( x_0^\prime | x_1 ) q_0 ( x_0 | x_1 ) p_1 (x_1) dx_0 dx_0^\prime \\
=& \int q_1 ( x_1^\prime | x_0^\prime ) q_0 ( x_0^\prime | x_1 ) \int q_0 ( x_0 | x_1 ) dx_0 \, p_1 (x_1) dx_0^\prime \\
=& \int q_1 ( x_1^\prime | x_0^\prime ) q_0 ( x_0^\prime | x_1 ) p_1 (x_1) dx_0^\prime \\
=& \int q_1 ( x_1^\prime | x_0^\prime ) p ( x_0^\prime , x_1) dx_0^\prime \\
=& \int q_1 ( x_1^\prime | x_0^\prime ) q_1 ( x_1 | x_0^\prime ) p_0 ( x_0^\prime ) dx_0^\prime \ .
\end{align*}
Since $q_1 ( x_1^\prime | x_0^\prime ) q_1 ( x_1 | x_0^\prime )$ is symmetric in $x_1$ and $x_1'$, this implies that 
$$
q (x_1^\prime ,x_1 ) = q ( x_1 , x_1^\prime ) \ ,
$$
which means that the Markov chain $X_{n,1}$ is reversible.
\end{proof}

\bibliographystyle{IEEEtranD}
\bibliography{References}

\end{document}